\documentclass{article}
\PassOptionsToPackage{numbers,compress}{natbib}

\usepackage[preprint]{neurips_2020}

\usepackage[utf8]{inputenc}
\usepackage[T1]{fontenc}
\usepackage{url}
\usepackage[bookmarks=false]{hyperref}
\usepackage[small]{caption}
\usepackage{booktabs}
\usepackage{amsfonts}
\usepackage{nicefrac}
\usepackage{microtype}
\usepackage{algorithm}
\usepackage{algorithmicx}
\usepackage[noend]{algpseudocode}
\usepackage[belowskip=-5pt,aboveskip=5pt]{caption}
\usepackage{wrapfig}
\usepackage{xspace}
\usepackage{dirtytalk}

\algnewcommand{\LineComment}[1]{\State\(\vartriangleright\) #1}

\title{Latent Bandits Revisited}

\usepackage{graphicx}
\usepackage{amsmath}
\usepackage{amssymb}
\usepackage{amsthm}
\usepackage{bbm}
\usepackage{bm}
\usepackage{dsfont}

\def\Sset{\mathcal{S}}
\def\Aset{\mathcal{A}}
\def\Xset{\mathcal{X}}

\renewcommand{\hat}{\widehat}

\newcommand{\indicator}[1]{\mathds{1} \! \left\{#1\right\}}
\newcommand{\E}[2]{\mathbb{E}_{#1} \left[#2\right]}
\newcommand{\prob}[1]{\mathbb{P}\left(#1\right)}

\newcommand{\abs}[1]{\left|#1\right|}

\newcommand{\mexp}{\ensuremath{\tt EXP4}\xspace}
\newcommand{\lints}{\ensuremath{\tt LinTS}\xspace}
\newcommand{\linucb}{\ensuremath{\tt LinUCB}\xspace}
\newcommand{\mucb}{\ensuremath{\tt mUCB}\xspace}
\newcommand{\mts}{\ensuremath{\tt mTS}\xspace}
\newcommand{\mmucb}{\ensuremath{\tt mmUCB}\xspace}
\newcommand{\mmts}{\ensuremath{\tt mmTS}\xspace}
\newcommand{\ts}{\ensuremath{\tt TS}\xspace}
\newcommand{\ucb}{\ensuremath{\tt UCB1}\xspace}

\def\Regret{\mathcal{R}}
\def\Bregret{\mathcal{BR}}

\usepackage[usenames,dvipsnames]{xcolor}
\hypersetup{
  pdffitwindow=true,
  pdfstartview={FitH},
  pdfnewwindow=true,
  colorlinks,
  linktocpage=true,
  linkcolor=Green,
  urlcolor=Green,
  citecolor=Green
}
\usepackage[capitalize,noabbrev]{cleveref}

\usepackage[backgroundcolor=white]{todonotes}

\newtheorem{lemma}{Lemma}
\newtheorem{corollary}{Corollary}
\newtheorem{theorem}{Theorem}

\author{
Joey Hong \\ 
Google Research \\
\texttt{jxihong@google.com} \\
\And
Branislav Kveton \\
Google Research \\
\texttt{bkveton@google.com} \\
\And
Manzil Zaheer \\
Google Research \\
\texttt{manzilzaheer@google.com} \\
\And
Yinlam Chow \\
Google Research \\
\texttt{yinlamchow@google.com} \\
\And
Amr Ahmed \\
Google Research \\
\texttt{amra@google.com} \\
\And
Craig Boutilier \\
Google Research \\
\texttt{cboutilier@google.com} \\
}

\begin{document}

\maketitle

\begin{abstract}
A \emph{latent bandit problem} is one in which the learning agent knows the arm reward distributions conditioned on an \emph{unknown discrete latent state}. The primary goal of the agent is to identify the latent state, after which it can act optimally. This setting is a natural midpoint between online and offline learning---complex models can be learned offline with the agent identifying latent state online---of practical relevance in, say, recommender systems.
In this work, we propose general algorithms for this setting, based on both upper confidence bounds (UCBs) and Thompson sampling. Our methods are contextual and  aware of model uncertainty and misspecification. We provide a unified theoretical analysis of our algorithms, which have lower regret than classic bandit policies when the number of latent states is smaller than actions. A comprehensive empirical study showcases the advantages of our approach.
\end{abstract}


\section{Introduction}

Many online platforms, such as search engines or recommender systems, display results based observed properties of the user and their query. However, a user's behavior is often influenced by \emph{latent state} not explicitly revealed to the system. This might be \emph{user intent} (e.g., reflecting a long-term task) in search, or \emph{user (short- and long-term) preferences} (e.g., reflecting topic interests) in a recommender. The unobserved latent state in each case influences the user response (hence, the associated reward) of the displayed results. A machine learning (ML) system, thus, should take steps to infer the latent state and tailor its results accordingly.

While many ML models use either heuristic features \cite{linucb,lints} or recurrent models \cite{rnn_recommender} to capture user history, explicit exploration for \emph{(latent) state identification} (i.e., reducing uncertainty regarding the true state) is less common in practice. In this paper, we study \emph{latent bandits}, which model online interactions of the type above. At each round, the learning agent is given an observed context (e.g., query, user demographics), selects an action (e.g., recommendation),
and observes its reward (e.g., user engagement with the recommendation). The action reward depends stochastically on both the context and the user latent state.
Hence the observed reward provides information about the unobserved latent state, which can be used to improve predictions at future rounds. We are interested in designing exploration policies that allow the agent to quickly maximize its per-round reward by resolving \emph{relevant} latent state uncertainty. Specifically, we want policies that have low \emph{$n$-round regret}.

Latent class structure of this form can allow an agent to quickly adapt it results to new users (e.g., cold start in recommenders) or adapt to new user tasks or intents on a per-session basis. For instance, clusters of users with similar item preferences can be used as the latent state of a new user. Estimated latent state can be used to quickly reach good cold-start recommendations if the number of clusters is much less than the number of items \citep{latent_contextual_bandits}. 


\emph{Fully} online exploration (e.g., for personalization) also involves learning a reward model---conditional on context and latent state---and generally requires massive amounts of interaction data.
Fortunately, many platforms have just such \emph{offline} data (e.g., past user interactions) with which to construct both a latent state space and reasonably accurate conditional reward models \citep{yahoo_contextual,msft_paper}. 
We assume such a model is available and focus on the simpler online problem of state identification. While previously studied, prior work on this problem  assumes the \emph{true} conditional reward models is given \citep{latent_bandits, latent_contextual_bandits}. Moreover, these algorithms are UCB-style, with optimal theoretical guarantees, but sub-par empirical performance. We provide a unified framework that combines offline-learned models with online exploration for both UCB and Thompson sampling algorithms, and propose practical, analyzable algorithms that are contextual and robust to natural forms of model imprecision.

Our main contributions are as follows. Our work is the first to propose algorithms that are aware of model uncertainty in the latent bandits setting. In Sec.~\ref{sec:algorithms}, we propose novel, practical algorithms based on UCB and Thompson sampling. Using a tight connection between UCB and posterior sampling \cite{russo_posterior_sampling}, 
we derive optimal theoretical bounds on the Bayes regret of our approaches in Sec.~\ref{sec:analysis}. Finally, in Sec.~\ref{sec:experiments}, we demonstrate its effectiveness vis-\'{a}-vis state-of-the-art benchmarks using both synthetic simulations and a large-scale real-world recommendation dataset.


\section{Problem Formulation}

We adopt the following notation. Random variables are capitalized. The set of arms is $\Aset = [K]$, the set of contexts is $\Xset$, and the set of latent states is $\Sset$, with $|\Sset| \ll K$. 

We study a {\em latent bandit} problem, where the learning agent interacts with an environment over $n$ rounds. In round $t \in [n]$, the agent observes context $X_t \in \Xset$, chooses action $A_t \in \Aset$, then observes reward $R_t \in \mathbb{R}$. The random variable $R_t$ depends on context $X_t$, action $A_t$, and latent state $s \in \Sset$, where $s$ is fixed but unknown.\footnote{The latent state $s$ can be viewed, say, as a user's current task or preferences, which is fixed over the course of a session or episode. The state is resampled (see below) for each user (or the same user at a future episode).} The \emph{observation history} up to round $t$ is $H_t = (X_1, A_1, R_1, \hdots, X_{t-1}, A_{t-1}, R_{t-1})$. An agent's \emph{policy} maps $H_t$ and $X_t$ to the choice of action $A_t$.

The reward is sampled from a \emph{conditional reward distribution}, $P(\cdot \mid A, X, s, \theta)$, which is parameterized by vector $\theta \in \Theta$, where $\Theta$ 
reflects the space of feasible reward models.
Let $\mu(a, x, s, \theta) = \E{R \sim P(\cdot \mid a, x, s, \theta)}{R}$ be the \emph{mean reward} of action $a$ in context $x$ and latent state $s$ under $\theta$. We denote the true (unknown) latent state by $s_*$ and true model parameters by $\theta_*$. These are generally \emph{estimated offline}. We assume that rewards are $\sigma^2$-sub-Gaussian with variance proxy $\sigma^2$: $\E{R \sim P(\cdot \mid a, x, s_*, \theta_*)}{\exp(\lambda (R - \mu(a, x, s_*, \theta_*)))} \leq \exp(\sigma^2 \lambda^2 / 2)$ for all $a$, $x$ and $\lambda > 0$. Note that we do not make strong assumptions about the form of the reward: $\mu(a, x, s, \theta)$ can be any complex function of $\theta$, and contexts generated by any arbitrary process.

We measure performance with regret.
For a fixed latent state $s_* \in \Sset$ and model $\theta_* \in \Theta$, let $A_{t, *} = \arg\max_{a \in \Aset} \mu(a, X_t, s_*, \theta_*)$ be the optimal arm. The \emph{expected $n$-round regret} is:
\begin{align}
    \Regret(n; s_*, \theta_*) = \E{}{\sum_{t=1}^n \mu(A_{t, *}, X_t, s_*, \theta_*) - \mu(A_t, X_t, s_*, \theta_*)}.
    \label{eqn:regret}
\end{align}
While fixed-state regret is useful, we are often more concerned with average performance over a range of states (e.g., multiple users, multiple sessions with the same user). Thus, we also consider Bayes regret, where we take expectation over latent-state randomness. Assuming $S_*$ and $\theta_*$ are drawn from some prior, the \emph{$n$-round Bayes regret} is:
\begin{align}
    \Bregret(n) 
    = \E{}{\Regret(n; S_*, \theta_*)}
    = \E{}{\sum_{t=1}^n \mu(A_{t, *}, X_t, S_*, \theta_*) - \mu(A_t, X_t, S_*, \theta_*)},
    \label{eqn:bayes_regret}
\end{align}
where $A_{t, *} = \arg\max_{a \in \Aset} \mu(a, X_t, S_*, \theta_*)$ additionally depends on random latent state and model.


\section{Algorithms}
\label{sec:algorithms}

In this section, we develop both UCB and Thompson sampling (TS) algorithms that leverage an environment model, generally learned offline, to expedite online exploration. As discussed above, such offline models can be readily learned given the large amounts of offline interaction data available to many interactive systems. 
In each subsection below, we specify a particular form of the offline-learned model, and develop a corresponding online algorithm.

\vspace{-0.05in}
\subsection{UCB with Perfect Model (\mucb)}
\vspace{-0.05in}

We first design a UCB-style algorithm that uses the learned model parameters $\hat{\theta} \in \Theta$. Let
$\hat{\mu}(a, x, s) = \mu(a, x, s, \hat{\theta})$ denote the estimated mean reward, and $\mu(a, x, s) = \mu(a, x, s, \theta_*) $ denote the true reward. We initially assume accurate knowledge of the true model, that is, we are given $\hat{\theta} = \theta_*$ as input.

The key idea in UCB algorithms is to compute high-probability upper confidence bounds $U_t(a)$ on the mean reward for each action $a$ in round $t$, where the $U_t$ is some function of history \citep{ucb}. UCB algorithms take action $A_t \!=\! \arg\max_{a \in \Aset} \! U_t(a)$. 
Our model-based algorithm \mucb (see Alg.~\ref{alg:ucb}) works in this fashion. It is similar to the method of \citet{latent_bandits}, but also handles context.

In round $t$, \mucb maintains a set of latent states $C_t$ that are \emph{consistent} with the rewards observed thus far. It chooses a specific (``believed'') latent state $B_t$ from the consistent set $C_t$ and
the arm $A_t$ with the maximum expected reward at that state: $(B_t, A_t) = \arg\max_{s \in C_t, a \in A} \hat{\mu}(a, X_t, s)$. Thus our UCB for $a$ is $U_t(a) = \arg\max_{s \in C_t} \hat{\mu}(a, X_t, s)$.
\mucb tracks two key quantities: the number of times $N_t(s)$ that state $s$ has been selected up to round $t$; and the \say{gap} $G_t(s)$ between the expected and realized rewards under $s$ up to round $t$ (see Eq.~\eqref{eqn:ucb_gap} in Alg.~\ref{alg:ucb}). 
If $G_t(s)$ is high, the algorithm marks $s$ as \emph{inconsistent} and does not consider it in round $t$. Notice that the gap is defined over latent states rather than over actions, and with respect to realized rewards rather than expected rewards. 




\begin{algorithm}[tb]
\caption{\mucb}\label{alg:ucb}
\begin{algorithmic}[1]
  \State \textbf{Input:} Model parameters $\hat{\theta}$
  \Statex
  \For{$t \gets 1, 2, \hdots$}
    \State Define $N_t(s) \leftarrow \sum_{\ell = 1}^{t-1}\indicator{B_\ell = s}$ and 
    \begin{align}
    \label{eqn:ucb_gap}
    G_{t}(s) \leftarrow \sum_{\ell = 1}^{t-1} \indicator{B_\ell = s}\left(\hat{\mu}(A_\ell, X_\ell, s) - R_\ell\right)
    \end{align} 
    \State Set of consistent latent states
    $C_t \leftarrow \left\{s \in S: G_t(s) \leq \sigma \sqrt{6N_t(s)\log n} \right\}$ 
    \State Select $B_t, A_t \leftarrow \arg\max_{s \in C_t, a \in A} \hat{\mu}(a, X_t, s)$
  \EndFor
\end{algorithmic}
\end{algorithm}

\vspace{-0.05in}
\subsection{UCB with Misspecified Model (\mmucb)}
\vspace{-0.05in}

We now generalize \mucb to handle a misspecified model, i.e.,  when we are given $\hat{\theta} \neq \theta_*$ as input. We formulate model misspecification assuming the following high-probability worst-case guarantee: there is a $\delta > 0$ such that $\abs{\hat{\mu}(a, x, s) - \mu(a, x, s)} \leq \varepsilon$ holds w.p. at least $1 - \delta$ jointly over all $a \in \Aset, x \in \Xset, s \in \Sset$. Guarantees of this form are, for example, offered by spectral learning methods for latent variable models, where $\varepsilon$ and $\delta$ are functions of the size of the offline dataset \citep{tensor_decomposition}.

We modify \mucb to be sensitive to this type of model error, deriving a new method \mmucb for misspecified models. We use the high-probability lower bound to rewrite the gap in Eq. \eqref{eqn:ucb_gap} as
\begin{align}
\label{eqn:ucb_gap_uncertain}
G_t(s) = \sum_{\ell = 1}^{t-1} \indicator{B_\ell = s}
  \left(\hat{\mu}(A_\ell, X_\ell, s) - \varepsilon - R_\ell\right).
\end{align}
This allows \mmucb to act conservatively when determining inconsistent latent states, so that  $s_* \in C_t$ occurs with high probability. Just as importantly, it is also useful for deriving worst-case regret bounds---we use it below to analyze TS algorithms with misspecified models.

\vspace{-0.05in}
\subsection{Thompson Sampling with Perfect Model (\mts)}
\vspace{-0.05in}

Our UCB-based algorithms \mucb and \mmucb are designed for worst-case performance. We now adopt an alternative perspective where, apart from the learned model parameters $\hat{\theta}$, we are given the conditional reward distribution $P(\cdot \mid a, x, s, \theta)$ for all  $a$,  $x$,  $s$ and  $\theta$, as well as a prior distribution over latent states $P_1$ as input. As above, we first assume $\hat{\theta} = \theta_*$.

TS samples actions according to their posterior probability (given history so far) of being optimal. Let the optimal action (w.r.t.\ the posterior) in round $t$ be $A_{t, *} = \arg\max_{a \in A} \mu(a, X_t, S_*, \theta_*)$, which is random due to the observed context and unknown latent state. TS selects $A_t$ stochastically s.t.\ $\prob{A_t = a \mid H_t} = \prob{A_{t, *} = a \mid H_t}$ for all $a$. An advantage of TS over UCB is that it obviates the need to design UCBs, which are often loose. Consequently, UCB algorithms are often conservative in practice and TS typically offers better empirical performance \citep{ts_empirical}.

Our latent-state TS method \mts, detailed in \cref{alg:thompson_1}, assumes an accurate model. For all $s \in \Sset$, let $P_t(s) = \prob{S_* = s \mid H_t}$ be the posterior probability that $s$ is the latent state in round $t$. In each round, \mts samples the latent state from the posterior $B_t \sim P_t$, and plays action $A_t = \max_{a \in \Aset} \hat{\mu}(a, X_t, B_t)$. Because $s$ is fixed, the posterior is $P_t(s) \propto P_1(s) \prod_{\ell=1}^{t-1} P(R_\ell \mid A_\ell, X_\ell, s, \hat{\theta})$, and $P_t$ can be updated incrementally in the standard Bayesian filtering fashion \cite{sarkka2013bayesian}.

\begin{figure}[tb]
\begin{minipage}[tb]{0.48\textwidth}
\begin{algorithm}[H]
\caption{\mts}\label{alg:thompson_1}
\begin{algorithmic}[1]
  \State \textbf{Input:}
  \State \quad Model parameters $\hat{\theta}$
  \State \quad Prior over latent states $P_1(s)$
  \Statex
  \For {$t \gets 1, 2, \hdots$}
    \State Define
    \begin{align*}\textstyle
    P_t(s)
    \propto P_1(s) \prod_{\ell=1}^{t-1} P(R_\ell \mid A_\ell, X_\ell, s, \hat{\theta})
    \end{align*} 
    \State Sample $B_t \sim P_t$
    \State Select 
    $A_t \leftarrow \arg\max_{a \in A} \hat{\mu}(a, X_t, B_t)$
\EndFor
\end{algorithmic}
\end{algorithm}
\end{minipage}
\hfill
\begin{minipage}[tb]{0.48\textwidth}
\begin{algorithm}[H]
\caption{\mmts}\label{alg:thompson_2}
\begin{algorithmic}[1]
  \State \textbf{Input:}
  \State \quad Prior over model parameters $P_1(\theta)$
  \State \quad Prior over latent states $P_1(s)$
  \Statex
  \For {$t \gets 1, 2, \hdots$}
    \State Define
    \begin{align*}\textstyle
    P_{t}(s, \theta) \propto P_1(s) P_1(\theta) \prod_{\ell = 1}^{t-1} P(R_\ell \mid A_\ell, X_\ell, s, \theta)
    \end{align*}
    \State Sample $B_t, \hat{\theta} \sim P_t$
    \State Select 
    $A_t \leftarrow \arg\max_{a \in A} \hat{\mu}(a, X_t, B_t)$
\EndFor
\end{algorithmic}
\end{algorithm}
\end{minipage}
\vspace{-0.1in}
\end{figure}

\vspace{-0.05in}
\subsection{Thompson Sampling with Misspecified Model (\mmts)} 
\vspace{-0.05in}

As in the UCB case, we also generalize our TS method \mts to handle a misspecified model. Instead of an estimated $\hat{\theta}$ with worst-case error as in \mmucb, we use a prior distribution $P_1(\theta)$ over possible models, and assume that $\theta_* \sim P_1$. This is well-motivated by prior literature on modeling epistemic uncertainty \cite{model_uncertainty}. In practice, learning a distribution over parameters is intractable for complex models, but approximate inference can be performed using, say, ensembles of bootstrapped models \cite{model_uncertainty}.

Our TS method \mmts (see Alg.~\ref{alg:thompson_2}) seamlessly integrates model uncertainty into \mts. At each round $t$, the latent state $B_t$ and estimated model parameters $\hat{\theta}$ are sampled from their joint posterior. Like \mts, the action is chosen to maximize $A_t = \max_{a \in \Aset} \hat{\mu}(a, X_t, B_t)$ using the sampled state and parameters. 
Approximate sampling from the posterior can be realized with sequential Monte Carlo methods \cite{smc}. 

When the model prior is conjugate to the likelihood, the posterior has a closed-form solution. Because $\Sset$ is finite, we can tractably sample from the joint posterior by first sampling latent state $B_t$ from its marginal posterior, then $\hat{\theta}$ conditioned on latent state $B_t$. For exponential family distributions, the posterior parameters can also be updated online and efficiently (see Appendix \ref{sec:mmts_specific} for details, and Appendix \ref{sec:mmts_pseudocode} for pseudocode for Gaussian prior and likelihood).


\section{Regret Analysis}
\label{sec:analysis}

\citet{latent_bandits} derive gap-dependent regret bounds for a UCB algorithm when the true model is known and arms are independent. We provide a unified analysis of our methods that extend their results to include context, model misspecification, and an analysis for TS.

\vspace{-0.05in}
\subsection{Regret Decomposition}
\vspace{-0.05in}

UCB algorithms explore using upper confidence bounds, while TS samples from the posterior. \citet{russo_posterior_sampling} relate these two classes of algorithms with a unified regret decomposition, showing how to analyze TS using UCB analysis. We adopt this approach.


Let $s_*$ be the true latent state. The regret of our UCB algorithms in round $t$ decomposes as
\begin{align*}
    \mu(A_{t, *}, X_t, s_*) - \mu(A_t, X_t, s_*)
    & = \mu(A_{t, *}, X_t, s_*) - U_t(A_t) + U_t(A_t) - \mu(A_t, X_t, s_*) \\
    & \leq \left[\mu(A_{t, *}, X_t, s_*) - U_t(A_{t, *})\right] +
    \left[U_t(A_t) - \mu(A_t, X_t, s_*)\right]\,,
\end{align*}
where the inequality holds by the definition of $A_t$. A similar inequality without latent states appears in prior work \citep{russo_posterior_sampling}. This yields the following regret decomposition:
\begin{align}
\begin{split}
    \Regret(n; s_*, \theta_*) 
    &\leq \E{}{\sum_{t=1}^n \mu(A_{t, *}, X_t, s_*) - U_t(A_{t, *})} + \E{}{\sum_{t=1}^n U_t(A_t) - \mu(A_t, X_t, s_*)}.
    \label{eqn:ucb_regret_decomposition}
\end{split}
\end{align}
An analogous decomposition exists for the Bayes regret of our TS algorithms. Specifically, for any TS algorithm and function $U_t$ of history, we have
\begin{align}
\begin{split}
    \Bregret(n) 
    &= \E{}{\sum_{t=1}^n \mu(A_{t, *}, X_t, S_*, \theta_*) - U_t(A_{t, *})} + 
    \E{}{ \sum_{t=1}^n U_t(A_t) - \mu(A_t, X_t, S_*, \theta_*)}.
    \label{eqn:posterior_regret_decomposition}
\end{split}
\end{align}
The proof uses the fact that $\E{}{U_t(A_{t, *}) \mid X_t, H_t} = \E{}{U_t(A_{t}) \mid X_t, H_t}$ holds for any $H_t$ and $X_t$ by definition of TS. Hence, $U_t$ can be the upper confidence bound of UCB algorithms. 

Though the UCBs $U_t$ are not used by TS algorithms, they can be used to \emph{analyze} TS due to Eq.~\eqref{eqn:posterior_regret_decomposition}. Thus regret bounds for UCB algorithms can be translated to Bayes regret bounds for TS.
We make two important points. First, we must use a worst-case argument over suboptimal actions when bounding the regret, since actions in TS do not maximize $U_t$. Second, because the Bayes regret is an expectation over states, the resulting regret bounds are problem-independent, i.e., gap-free.

\vspace{-0.05in}
\subsection{Key Steps in Our Proofs}
\vspace{-0.05in}

Full proofs of our unified regret analyses can be found in the appendix. All proofs follow the same outline, the key steps of which are outlined below. To ease the exposition, we assume the suboptimality of any action is bounded by $1$.

\textbf{Step 1: Concentration of realized rewards at their means.}
We first show that the total observed reward does not deviate too much from its expectation, under any latent state $s$. 
Formally, we show
$
\prob{
    \abs{\sum_{\ell = 1}^{t-1} \indicator{B_\ell = s}
    \left(\mu(A_\ell, X_\ell, s_*) - R_\ell\right)} \geq \sigma \sqrt{6N_t(s) \log n}} = O(n^{-2})
$
for any round $t$ and latent state $s \in \Sset$.
When the arms are independent, as in prior work, this follows from Hoeffding's inequality. However, we also consider the case of contextual arms, which requires joint estimators over dependent arms. To address this, we resort to martingales and Azuma's inequality.

\textbf{Step 2: $s_* \in C_t$ in each round $t$ with a high probability.} We show that our consistent sets are unlikely to rule out the true latent state. This follows from the concentration argument in Step 1, for $s = s_\ast$. Then, in any round $t$ where $s_* \in C_t$, we use that $U_t(a) \geq \mu(a, X_t, s_*)$ for any  $a$ for \mucb, or $U_t(a) \geq \mu(a, X_t, s_*) - \varepsilon$ for \mmucb.

\textbf{Step 3: Upper bound on the UCB regret.} This bound is proved by bounding each term in the regret decomposition in Eq.~\eqref{eqn:ucb_regret_decomposition}. By Steps 1-2, the first term is at most $0$ with high probability. The second term is the sum over rounds of confidence widths, or difference between $U_t$ and the true expected mean reward at $t$. We partition this sum by the latent state selected at each round. For each $s$, we almost have an upper bound on the its sum, excluding the last round it is played, via $G_n(s)$,
\begin{align*}
    \sum_{t = 1}^{n} \indicator{B_t = s} \left(U_t(A_t) - \mu(A_t, X_t, s)\right)
    &= (G_n(s) + 1) + \sum_{t = 1}^{n} \indicator{B_t = s} \left(R_t - \mu(A_t, X_t, s)\right).
\end{align*}
If $s$ is chosen in round $t$, we know $G_t(s) \leq \sigma\sqrt{6N_t(s) \log n}$. The other term is bounded by Step 1, which gives a $2\sigma\sqrt{6N_n(s) \log n}$ total upper-bound. We combine the bounds for $s$'s partition with the Cauchy-Schwarz inequality.


\textbf{Step 4: Upper bound on the TS regret.} We exploit the fact that the regret decomposition for Bayes regret in Eq.~\eqref{eqn:posterior_regret_decomposition} is the same as that for the UCB regret in Eq.~\eqref{eqn:ucb_regret_decomposition}. Because our UCB analysis is worst-case over suboptimal latent states and actions, and gap-free, any regret bound transfers immediately to the Bayes regret bound for TS.

\vspace{-0.05in}
\subsection{Regret Bounds}
\vspace{-0.05in}

Our first result is an upper bound on the $n$-round regret of \mucb when the true model is known. This result differs from that of \citet{latent_bandits} in two respects: our bound is gap-free and accounts for context.

\begin{theorem}
\label{thm:ucb_regret} Assume that $\hat{\theta} = \theta_*$. Then, for any $s_* \in \Sset$ and $\theta_\ast \in \Theta$, the $n$-round regret of \mucb is bounded as $\Regret(n; s_*, \theta_*) \leq 3|\Sset| +  2\sigma \sqrt{6|\Sset| n \log n}$.
\end{theorem}

A gap-free lower bound on regret in multi-armed bandits with independent arms is $\Omega(\sqrt{K n})$ \cite{exp4}. Our upper bound is optimal up to log factors, but substitutes actions $\Aset$ with latent states $\Sset$ and includes context. Our bound can be much lower when $|\Sset| \ll K$, and holds for arbitrarily complex reward models. Using Step 4 of the proof outline, we also have that the Bayes regret of \mts is bounded:

\begin{corollary}
\label{cor:posterior_regret} Assume that $\hat{\theta} = \theta_*$. Then, for $S_* \sim P_1$ and any $\theta_* \in \Theta$, the $n$-round Bayes regret of \mts is bounded as $\Bregret(n) \leq 3|\Sset| +  2\sigma \sqrt{6|\Sset| n \log n}$.
\end{corollary}

Our next results apply to the cases with misspecified models. We assume $\hat{\theta}$ was estimated by some black-box method. For 
\mmucb, our regret bound depends on the high-probability maximum error $\varepsilon$.

\begin{theorem}
\label{thm:ucb_regret_uncertain} Let $\prob{\forall a \in \Aset, x \in \Xset, s \in \Sset: |\mu(a, x, s, \hat{\theta}) - \mu(a, x, s, \theta_*)| \leq \varepsilon} \geq 1 - \delta$ for some $\varepsilon, \delta > 0$. Then, for any $s_* \in \Sset$ and $\theta_* \in \Theta$, the $n$-round regret of \mmucb is bounded as
\begin{align*}
  \Regret(n; s_*, \theta_*)
  \leq n\delta + 3|\Sset| + 
  2 n \varepsilon +  2\sigma \sqrt{6|\Sset| n \log n}\,.
\end{align*}
\end{theorem}

The proof of \cref{thm:ucb_regret_uncertain} follows the same proof outline. Steps 1--2 are unchanged, but bounding the regret decomposition in Step 3 requires accounting for the error due to model misspecification. The linear dependence on $\varepsilon$ and probability $\delta$ is unavoidable in the worst-case, specifically if $\varepsilon$ is larger than the suboptimality gap. However, some offline model-learning methods, i.e. tensor decomposition \citep{tensor_decomposition}, allow for $\varepsilon, \delta$ to be arbitrarily small as size of offline dataset increases.

For \mmts, we assume that a prior distribution over model parameters is known. 
Instead of $\hat{\mu}(a, x, s)$ due to a single $\hat{\theta}$, we define $\bar{\mu}(a, x, s) = \int_\theta \mu(a, x, s, \theta) P_1(\theta) d \theta$ as the mean conditional reward, marginalized with respect to the prior. We obtain the following Bayes regret bound:

\begin{corollary}
\label{cor:posterior_regret_uncertain}
For $\theta_* \sim P_1$, let $\prob{\forall a \in \Aset, x \in \Xset, s \in \Sset: |\bar{\mu}(a, x, s) - \mu(a, x, s, \theta_*)| \leq \varepsilon} \geq 1 - \delta$ for some $\varepsilon, \delta > 0$. Then, for $S_*, \theta_* \sim P_1$, the $n$-round Bayes regret of \mmts is bounded as
\begin{align*}
    \Bregret(n)
    \leq  n\delta + 3|\Sset| +
    2 n \varepsilon +  2\sigma \sqrt{6|\Sset| n \log n}\,.
\end{align*}
\end{corollary}

We can formally define $\varepsilon$ and $\delta$ in terms of the tails of the conditional reward distributions. Let $\mu(a, x, s, \theta) - \bar{\mu}(a, x, s)$ be $v^2$-sub-Gaussian in $\theta \sim P_1$ for all $a$, $x$, and $s$. For $\delta > 0$, choosing $\varepsilon = O(\sqrt{2v \log(K|\Xset||\Sset|/\delta}))$ satisfies the conditions on $\varepsilon$ and $\delta$ needed for \cref{cor:posterior_regret_uncertain}.
The proof uses $U_t$ in Eq. \eqref{eqn:posterior_regret_decomposition} as $U_t(a) = \arg\max_{s \in C_t}\bar{\mu}(a, X_t, s)$, i.e., quantities in \mmucb are defined using the marginalized conditional means instead of means using a point estimate $\hat{\theta}$.


\section{Experiments}
\label{sec:experiments}

In this section, we evaluate our algorithms on both synthetic and real-world datasets.
We compare the following methods: (i) \textbf{UCB}: \ucb/\linucb with no offline model \citep{ucb,linucb}; (ii) \textbf{TS}: \ts/\lints with no offline \citep{ts,lints}; (iii) \textbf{EXP4}: \mexp using offline reward model as experts \citep{exp4} (iv) \textbf{mUCB, mmUCB}: our proposed UCB algorithms \mucb and \mmucb; (v) \textbf{mTS, mmTS}: our proposed TS algorithms \mts and \mmts.
In contrast to our methods, the UCB and TS baselines do not use an offline learned model. \ucb and \ts are used for non-contextual problems, while \linucb and \lints are used for contextual bandit experiments. \mexp uses the offline-learned model as a mixture-of-experts, where each expert plays the best arm given context under its corresponding latent state. Because we measure ``fast personalization," we use short horizons of at most $500$.

\begin{figure}
\begin{minipage}{0.33\textwidth}
    \centering
    \includegraphics[width=\linewidth]{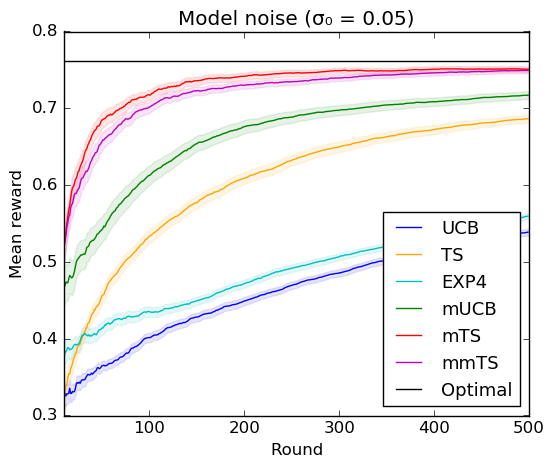}
\end{minipage}
\begin{minipage}{0.33\linewidth}
    \centering
    \includegraphics[width=\linewidth]{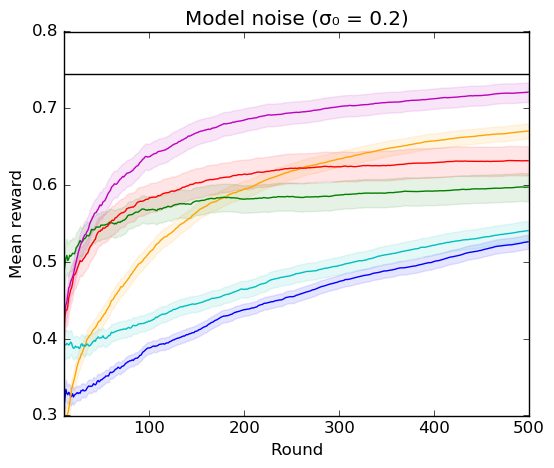}
\end{minipage}
\begin{minipage}{0.33\linewidth}
    \centering
    \includegraphics[width=\linewidth]{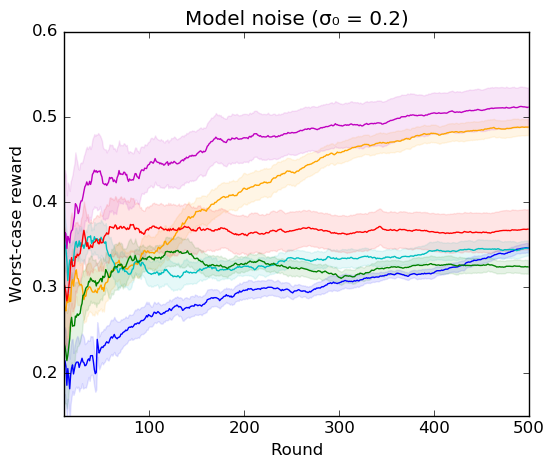}
\end{minipage}
\caption{Left: Mean reward and standard error in simulation for small model noise ($\sigma_0 = 0.05$). Middle/Right: Mean/worst-case reward and standard error for large model noise ($\sigma_0 = 0.2$).}
\label{fig:sim_results}
\vspace{-0.1in}
\end{figure}

\vspace{-0.05in}
\subsection{Synthetic Experiments}
\vspace{-0.05in}

We first experiment with synthetic (non-) multi-armed bandits with $\Aset = [10]$ and $\Sset = [5]$. Mean rewards are sampled uniformly at random $\mu(a, s) \sim \mathsf{Uniform}(0, 1)$ for each $a \in \Aset, s \in \Sset$. Using rejection sampling, we constrain the suboptimality gap of all actions to be at least $0.1$ at each $s$ to ensure significant comparisons between methods on short timescales. 
Observed rewards are drawn i.i.d.\ from $P(\cdot \mid a, s) = \mathcal{N}(\mu(a, s), \sigma^2)$ with $\sigma = 0.5$. We evaluate each algorithm using $100$ independent runs with a uniformly sampled latent state, and report average reward over time. 
We analyze the effect of model misspecification by perturbing the reward means with various degrees of noise:
given noise $\sigma_0$, estimated means are sampled from $\hat{\mu}(a, s) \sim \mathcal{N}(\mu(a, s), \sigma_0^2)$ for each arm and latent state. The estimated reward model $\hat{\theta}$ is the concatenation of all estimated means.

The leftmost plot in Fig.~\ref{fig:sim_results} shows average reward obtained over time when model noise  $\sigma_0 = 0.05$ is small. The middle plot increases noise to $\sigma_0 = 0.2$. Our algorithms \mucb and \mts perform much better than baselines \ucb and \ts when model noise is low, but degrade with higher noise, since neither accounts for model error. By contrast, \mmts outperforms \mts in the high-noise setting. However, \mmucb (not reported in the plot to reduce clutter) performs the same as \mucb; this is likely due to the conservative nature of UCB. Though having similar worst-case guarantees, \mexp performs poorly, suggesting that our algorithms generally use the offline model more intelligently.

The rightmost plot in Fig.~\ref{fig:sim_results} is the same as the middle one, but shows the ``worst-case'' performance by averaging the $10\%$ of runs, where the final reward of each method is lowest. Baselines \ucb and \ts are unaffected by model misspecification, and have better worst-case performance than \mucb and \mts. However, \mmts beats both online baselines; this demonstrates that uncertainty-awareness makes our algorithms more robust to model misspecification or learning error.

\vspace{-0.05in}
\subsection{MovieLens Results}
\vspace{-0.05in}

We also assess the empirical performance of our algorithms on MovieLens 1M \citep{movielens}, a large-scale, collaborative filtering dataset, comprising 6040 users rating 3883 movies. Each movie has a set of genres. We filter the data to include only users who rated at least 200 movies, and movies rated by at least 200 users, resulting in 1353 users and 1124 movies.

We randomly select $50\%$ of all ratings as our ``offline" training set, and use the remaining 50$\%$ as a test set, giving sparse ratings matrices $M_{\text{train}}$ and $M_{\text{test}}$.
We complete each matrix using least-squares matrix completion \citep{pmf} with rank $20$. We chose rank to be expressive enough to yield low prediction error, but small enough to not overfit. The learned factors are $M_{\text{train}} = \hat{U} \hat{V}^T$ and $M_{\text{test}} = U V^T$. User $i$ and movie $j$ correspond to row $U_i$ and $V_j$, respectively, in the matrix factors.

We define a latent contextual bandit instance with $\Aset = [20]$ and $\Sset = [5]$ as follows. Using $k$-means on rows of $\hat{U}$, we cluster users into $5$ clusters, where $5$ is the largest value that does not yield empty clusters. First, a user $i$ is sampled at uniformly at random. At each round, $20$ genres, then a movie for each genre, are uniformly sampled, creating a set of diverse movies.
Context $x_t \in \mathbb{R}^{20 \times 20}$ is the matrix with training movie vectors for the $20$ sampled movies as rows, i.e., movie $j$ has vector $\hat{V}_j$. The agent chooses among movies in $x_t$. The reward distribution $\mathcal{N}(U_i^T V_j, 0.5)$ for movie $j$ under user $i$ has the product of the test user and movie vectors as its mean. We evaluate on $100$ users.

Let $\hat{\theta}$ be the mean of the cluster. We assume a Gaussian prior over parameters with mean $\hat{\theta}$ and use the empirical covariance of user factors within each cluster as its covariance. Notice that baselines \linucb and \lints are also given movie vectors from the training set via context, and need to only learn the user vector. This is more information than low-rank bandit algorithms \cite{clustering_bandits_1}, which jointly learn user and movie representations, and are unlikely to converge on the short timescales we consider.

The left plot in Fig.~\ref{fig:movielens} show the mean rating and standard error of the six algorithms (as above, \mmucb is similar to \mucb and is not shown). 
\mucb and \mts adapt or ``personalize'' to users more quickly than \linucb and \lints even with access to movie vectors, and converge to better policies than \mexp. Despite this, both \mucb and \mts are affected by model misspecification. By contrast, \mmts handles model uncertainty and converges to the best reward. The right plot in Fig.~\ref{fig:movielens}
shows average results for the bottom $10\%$ of users. Again, \mmts dramatically outperforms \mts in the worst-case. 

\begin{figure}
\centering
\begin{minipage}{0.33\textwidth}
    \centering
    \includegraphics[width=\linewidth]{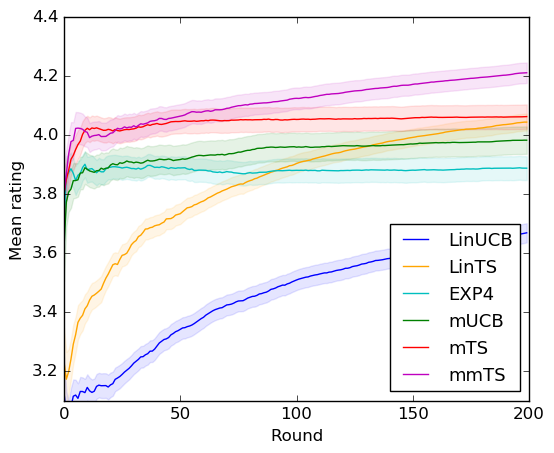}
\end{minipage}
\begin{minipage}{0.33\linewidth}
    \centering
    \includegraphics[width=\linewidth]{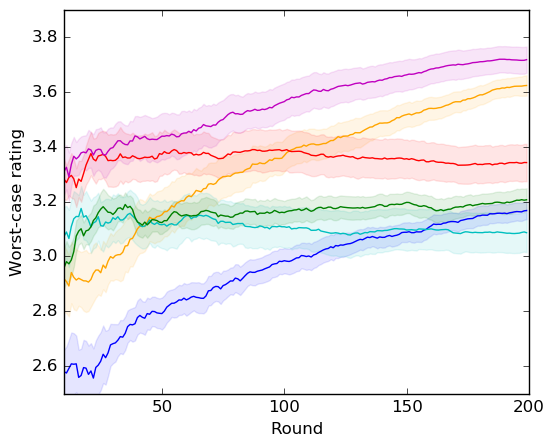}
\end{minipage}
\caption{Mean/worst-case rating and standard error on MovieLens 1M.}
\label{fig:movielens}
\vspace{-0.1in}
\end{figure}


\section{Related Work}

\textbf{Latent bandits.}
Latent contextual bandits admit faster personalization than standard contextual bandit strategies, such as LinUCB ~\citep{linucb} or linear TS \citep{lints,lints_2}. The closest work to ours is that of \citet{latent_bandits}, which proposes and analyzes non-contextual UCB algorithms under the assumption that the mean rewards for each latent state are known. \citet{latent_contextual_bandits} extend this formulation to the contextual bandits case, but consider offline-learned policies deployed as a mixture via EXP4. Bayesian policy reuse (BPR) \citep{bpr} selects offline-learned policies by maintaining a belief over the optimality of each policy, but no analysis exists. Our work subsumes prior work by providing contextual, uncertainty-aware UCB and TS algorithms and a unified analysis of the two.

\textbf{Low-rank bandits.} 
Low-rank bandits can be viewed as a generalization of latent bandits, where low-rank matrices that parameterize the reward are learned jointly with bandit strategies. \citet{ts_online_mf} propose a TS algorithm for low-rank matrix factorization; however, their algorithm is inefficient and analysis is provided only for the rank-$1$ case. \citet{latent_confounders} analyze an $\varepsilon$-greedy algorithm, but rely on properties that rarely hold in practice. Another body of work studies online clustering of bandit instances, which is based on a more specific low-rank structure \citep{clustering_bandits_1,clustering_bandits_2, clustering_bandits_3, clustering_bandits_4}. Yet another deals with low-rank matrices where both rows and columns are arms \citep{bandits_rank_1,bernoulli_bandits_rank_1}. None of this existing work leverages models that are learned offline---an important practical consideration given the general availability of offline data---and only linear reward models are learned. In \cref{sec:experiments}, we compare against idealized versions of these methods where low-rank features are provided.

\textbf{Structured bandits.}
In structured bandits, arms are related by a common latent parameter. \citet{structured_bandits} propose a UCB algorithm for the multi-arm setting. Recently, \citet{unified_structured_bandits} propose a unified framework that adapts classic bandit algorithms, such as UCB and TS, to the multi-arm structured bandit setting. Though similar to our work, the algorithms proposed differ in key aspects: we track confidence intervals around latent states instead of arms, and develop contextual algorithms that are robust to model (parameter) misspecification.



\section{Conclusions}

In this work, we studied the latent bandits problem, where the rewards are parameterized by a discrete, latent state. We adopted a framework in which an offline-learned model is combined with UCB and Thompson sampling exploration to quickly identify the latent state. Our approach handles both context and misspecified models. We analyzed our proposed algorithms using a unified framework, and validated them using both synthetic data and the MovieLens 1M dataset.
A natural extension of our work is to use temporal models to handle latent state dynamics. This is useful for applications where user preferences, tasks or intents change fairly quickly. For UCB, we can leverage existing adaptations to UCB algorithms (e.g., discounting, sliding windows). ~\citep{nonstationary_ucb}. For TS, we can take the dynamics into the account when computing the posterior. 



\bibliographystyle{named}
\bibliography{neurips_2020}

\clearpage
\onecolumn
\appendix


\section{Details of \mmts for Exponential Families}
\label{sec:mmts_specific}
For a matrix (vector) $M$, we let $M_i$ denote its $i$-th row (element). Using this notation, we can write $\theta = (\theta_s)_{s \in \Sset}$ as a vector of parameters, one for each latent state; each $\theta_s$ parameterizes the reward under latent state $s$.
We want to show that the sampling step in \mmts can be done tractably when the conditional reward distribution and model prior are in the exponential family.

We can write the conditional reward likelihood as,
\begin{align*}
    P(r \mid a, x, \theta, s) = \exp\left[\phi(r, a, x)^\top \kappa(\theta_s) - g(\theta_s) \right],
\end{align*}
where $\phi(r, a, x)$ are sufficient statistics for the observed data, $\kappa(\theta_s)$ are the natural parameters, and $g(\theta_s) = \log \sum_{r, a, x}\phi(r, a, x)^\top \kappa(\theta_s)$ is the log-partition function. 
Then, we assume the prior over $\theta_s$ to be the conjugate prior of the likelihood, which will have the general form,
\begin{align*}
    P_1(\theta_s) = H(\phi_0, m_0)\exp\left[\phi_0^\top \kappa(\theta_s) - m_0 g(\theta_s) \right],
\end{align*}
where $\phi_0, m_0$ are parameters controlling the prior, and $H(\phi_0, m_0)$ is the normalizing factor.

For round $t$, recall that $N_t(s) = \sum_{\ell = 1}^{t - 1} \indicator{B_\ell = s}$ is the number of times $s$ is selected. We can write the joint posterior as,
\begin{align}
P_t(s, \theta) 
&\propto
P_1(s) P_1(\theta_s)
\prod_{\ell = 1}^{t-1} 
\exp\left[\phi(R_\ell, A_\ell, X_\ell)^\top \kappa(\theta_s) - g(\theta_s)\right]^{\indicator{B_\ell = s}} \label{eqn:exponential_joint_posterior} \\
&\propto P_1(s)
\exp\left[\left(\phi_0 + \sum_{\ell = 1}^{t-1}\indicator{B_\ell = s}\phi(R_\ell, A_\ell, X_\ell) \right)^\top \kappa(\theta_s) - (m_0 + N_t(s))g(\theta_s) \right]. \nonumber
\end{align}
The general form for an exponential family likelihood is still retained. The prior-to-posterior conversion simply involves updating the prior parameters with sufficient statistics from the data. Specifically, updated parameters $\phi_t \leftarrow \phi_0 + \sum_\ell \indicator{B_\ell = s}\phi(R_\ell, A_\ell, X_\ell)$ and $m_t \leftarrow m_0 + N_t(s)$ form the conditional posterior
$
P_t(\theta_s) = H(\phi_t, m_t)\exp\left[\phi_t^\top \kappa(\theta_s) - m_t g(\theta_s) \right]
$.

For round $t$, the marginal posterior of $s$ is given by,
\begin{align*}
 P_t(s) 
  &\propto P_1(s)
  \int_\theta P_1(\theta_s)
  \exp\left[\phi_t^\top \kappa(\theta_s) - m_t g(\theta_s) \right]
  d \theta \\
  &\propto P_1(s) H(\phi_t, m_t).
\end{align*}
So, for all states $s$, and parameters $\theta$, the posterior probabilities $P_t(s)$ and $P_t(\theta_s)$ have analytic, closed-form solutions. Thus, sampling from the joint posterior can be done tractably by sampling state $s$ from its marginal posterior, then parameters $\theta_s$ from its conditional posterior.

\section{Pseudocode of \mmts for Gaussians}
\label{sec:mmts_pseudocode}
Next, we provide specific variants of \mmts when both the model prior and conditional reward likelihood are Gaussian. This is a common assumption for Thompson sampling algorithms \cite{ts,lints,lints_2}. In this case, the joint posterior in Eq. \eqref{eqn:exponential_joint_posterior} consists of Gaussians. We adopt the notation that $\mathcal{N}(r \mid \mu, \sigma^2) \propto \exp[- (r - \mu)^2 / 2 \sigma^2]$ is the Gaussian likelihood of $r$ given mean $\mu$ and variance $\sigma^2$.

We detail algorithms for two cases: \cref{alg:thompson_gauss} is for a multi-armed bandit with independent arms (no context), and \cref{alg:thompson_lingauss} is for a linear bandit problem. In the first case, we have that $\theta_s \in \mathbb{R}^K$ are the mean reward vectors where $\theta_{s,a} = \mu(a, s, \theta)$. 
In the other case, we assume that context is given by $x \in \mathbb{R}^{K \times d}$ where $x_a \in \mathbb{R}^d$ is the feature vector for arm $a$. Then, we have that $\theta_s \in \mathbb{R}^d$ are rank-$d$ vectors such that $x_a^\top\theta_s = \mu(a, x, s, \theta)$. Both algorithms are efficient to implement, and perform exact sampling from the joint posterior.

\begin{algorithm}[H]
\caption{Independent Gaussian \mmts (Non-contextual)}\label{alg:thompson_gauss}
\begin{algorithmic}[1]
  \State \textbf{Input:}
  \State \quad Prior over model parameters $P_1(\theta_s) = \mathcal{N}(\bar{\theta}_s, \sigma_0^2I), \forall s \in \Sset$
  \State \quad Prior over latent states $P_1(s)$
  \Statex
  \For {$t \gets 1, 2, \hdots$}
    \LineComment{Step 1: sample latent state from marginal posterior.}
    \State Define
    \begin{align*}
        P_t(s)
        \propto P_1(s) \prod_{\ell = 1}^{t-1} \mathcal{N}(R_\ell \mid \bar{\theta}_{s, A_\ell}, \sigma_0^2 + \sigma^2)^{\indicator{B_\ell = s}}
    \end{align*}
    \State Sample $B_t \sim P_t$
    \Statex\LineComment{Step 2: sample model parameters from conditional posteriors.}
    \State Define
    \begin{align*}
    N_t(a, s) \leftarrow \sum_{\ell = 1}^{t-1} \indicator{A_\ell = a, B_\ell = s},
    \text{ and } \quad
    S_t(a, s) \leftarrow \sum_{\ell = 1}^{t-1} \indicator{A_\ell = a, B_\ell = s} R_\ell
    \end{align*}
    \State For each $s \in \Sset$, sample $\hat{\theta}_s \sim \mathcal{N}(M_s, \mathsf{diag}(K_s))$, where
    \begin{align*}
    K_{s,a} \leftarrow \left(\sigma_0^{-2} + N_t(a, s) \sigma^{-2}\right)^{-1},
    \text{ and } \quad 
     M_{s,a} \leftarrow
     K_{s,a} \left(\sigma_0^{-2}\bar{\theta}_{s, a} + \sigma^{-2}S_t(a, s) \right)
    \end{align*}
    \State Select 
    $A_t \leftarrow \arg\max_{a \in A} \hat{\theta}_{B_t, a}$
\EndFor
\end{algorithmic}
\end{algorithm}

\begin{algorithm}[H]
\caption{Linear Gaussian \mmts}\label{alg:thompson_lingauss}
\begin{algorithmic}[1]
  \State \textbf{Input:}
  \State \quad Prior over model parameters $P_1(\theta_s) = \mathcal{N}(\bar{\theta}_s, \Sigma_0), \forall s \in \Sset$
  \State \quad Prior over latent states $P_1(s)$
  \Statex
  \For {$t \gets 1, 2, \hdots$}
    \LineComment{Step 1: sample latent state from marginal posterior.}
    \State Define
    \begin{align*}
        P_t(s)
        \propto P_1(s) \prod_{\ell = 1}^{t-1}
        \mathcal{N}(R_\ell \mid X_{\ell, A_\ell}^\top \bar{\theta}_{s}, \, X_{\ell, A_\ell}^\top \Sigma_0^{-1}X_{\ell, A_\ell} + \sigma^2)^{\indicator{B_\ell = s}}
    \end{align*}
    \State Sample $B_t \sim P_t$
    \Statex\LineComment{Step 2: sample model parameters from conditional posteriors.}
    \State Define $N_t(s) \leftarrow \sum_{\ell = 1}^{t-1} \indicator{B_\ell = s}$,
    \begin{align*}
    S_t(s) \leftarrow I + \sum_{\ell = 1}^{t-1} \indicator{B_\ell = s} X_{\ell, A_\ell} X_{\ell, A_\ell}^\top,
    \text{ and } \quad 
    F_t(s) \leftarrow \sum_{\ell = 1}^{t-1} \indicator{B_\ell = s} X_{\ell, A_\ell}R_\ell
    \end{align*}
    \State For each $s \in \Sset$, compute $\hat{\beta}_s \leftarrow S_t(s)^{-1}F_t(s)$, and $\hat{\Sigma}_s \leftarrow \sigma^2 S_t(s)^{-1}$
    \State For each $s \in \Sset$, sample $\hat{\theta}_s \sim \mathcal{N}(M_s, K_s)$, where
    \begin{align*}
     K_s \leftarrow \left(\Sigma_0^{-1} + N_t(s) \hat{\Sigma}_s^{-1} \right)^{-1},
    \text{ and } \quad 
     M_s \leftarrow
     K_s \left(\Sigma_0^{-1}\bar{\theta}_s + N_t(s) \hat{\Sigma}_s^{-1} \hat{\beta}_s  \right)
    \end{align*}
    \State Select 
    $A_t \leftarrow \arg\max_{a \in A} X_{\ell, a}^\top \hat{\theta}_{B_t}$
\EndFor
\end{algorithmic}
\end{algorithm}

\newpage
\section{Proofs}


Our proofs rely on the following concentration inequality, which is a straightforward extension of the Azuma-Hoeffding inequality to sub-Gaussian random variables.

\begin{lemma}
\label{thm:azuma_general}
Let $(Y_t)_{t \in [n]}$ be a martingale difference sequence with respect to filtration $(\mathcal{F}_t)_{t \in [n]}$, that is $\E{}{Y_t \mid \mathcal{F}_{t - 1}} = 0$ for any $t \in [n]$. Let $Y_t \mid \mathcal{F}_{t - 1}$ be $\sigma^2$-sub-Gaussian for any $t \in [n]$. Then for any $\varepsilon > 0$, 
\begin{align*}
  \prob{\Big|\sum_{t = 1}^n Y_t\Big| \geq \varepsilon}
  \leq 2 \exp\left[- \frac{\varepsilon^2}{2 n \sigma^2}\right]\,.
\end{align*}
\end{lemma}
\begin{proof}
For any $\lambda > 0$, which we tune later, we have
\begin{align*}
  \prob{\sum_{t = 1}^n Y_t \geq \varepsilon}
  = \prob{\prod_{t = 1}^n e^{\lambda Y_t} \geq e^{\lambda \varepsilon}}
  \leq e^{- \lambda \varepsilon} \E{}{\prod_{t = 1}^n e^{\lambda Y_t}}\,.
\end{align*}
The inequality is by Markov's inequality. From the conditional independence of $Y_t$ given $\mathcal{F}_{t - 1}$, the right term becomes
\begin{align*}
  \E{}{\prod_{t = 1}^n e^{\lambda Y_t}}
  = \E{}{\E{}{e^{\lambda Y_n} \mid \mathcal{F}_{n-1}}
  \prod_{t = 1}^{n - 1} e^{\lambda Y_t}}
  \leq e^{\frac{\lambda^2 \sigma^2}{2}}
  \E{}{\prod_{t = 1}^{n - 1} e^{\lambda Y_t}}
  \leq e^{\frac{n \lambda^2 \sigma^2}{2}}\,.
\end{align*}
We use that $Y_n \mid \mathcal{F}_{n - 1}$ is $\sigma^2$-sub-Gaussian in the first inequality, and recursively repeat for all rounds in the second. So we have
\begin{align*}
  \prob{\sum_{t = 1}^n Y_t \geq \varepsilon}
  \leq \min_{\lambda > 0} e^{-\lambda \varepsilon + \frac{n\lambda^2 \sigma^2}{2}} \,.
\end{align*}
The minimum is achieved at $\lambda = \varepsilon / (n \sigma^2)$. Therefore
\begin{align*}
  \prob{\sum_{t = 1}^n Y_t \geq \varepsilon}
  \leq \exp\left[- \frac{\varepsilon^2}{2 n \sigma^2}\right]\,.
\end{align*}
Now we apply the same proof to $\prob{- \sum_{t = 1}^n Y_t \geq \varepsilon}$, which yields a multiplicative factor of $2$ in the upper bound. This concludes the proof.
\end{proof}

\subsection{Proof of \cref{thm:ucb_regret}}
\label{sec:thm1_proof}

Recall that $s_* \in \Sset, \theta_* \in \Theta$ are the true latent state and model. Let $\mu(a, x) = \mu(a, x, s_*, \theta_*)$ be the true mean rewards given observed context and action.
Let
\begin{align}
    E_t = \left\{
    \forall s \in \Sset: \, 
    \abs{\sum_{\ell = 1}^{t-1} 
    \mathbbm{1}\{B_\ell = s\} \left(\mu(A_\ell, X_\ell) - R_\ell\right)} \leq \sigma \sqrt{6N_t(s) \log n} \right\}
    \label{eqn:ucb_event}
\end{align} 
be the event that the total realized reward under each played latent state is close to its expectation. Let $E = \cap_{t=1}^n E_t$ be the event that this holds for all rounds, and $\bar{E}$ be its complement.
We can rewrite the expected $n$-round regret by
\begin{align}
\begin{split}
\label{eqn:regret_event_decomposition}
    \Regret(n)
    &= \E{}{\indicator{\bar{E}} \Regret(n)} + 
    \E{}{\indicator{E} \Regret(n)} \\
    &\leq  \E{}{\indicator{\bar{E}} \sum_{t = 1}^n \mu(A_{t, *}, X_t) - \mu(A_t, X_t)}  \\
    &\quad + 
    \E{}{\indicator{E}\sum_{t = 1}^n \left(\mu(A_{t, *}, X_t) - U_t(A_{t, *})\right)} + \E{}{\indicator{E} \sum_{t = 1}^n \left(U_t(A_t) - \mu(A_t, X_t)\right)}\,, \hspace{-0.1in}
\end{split}
\end{align} 
where we use the regret decomposition in Eq. \eqref{eqn:ucb_regret_decomposition} in the inequality. 

Our first lemma is that the probability of $\bar{E}$ occurring is low. Without context, the lemma would follow immediately from Hoeffding's inequality. Since we have context generated by some random process, we instead turn to martingales. 

\begin{lemma}
\label{thm:ucb_concentration} Let $E_t$ be defined as in Eq. \eqref{eqn:ucb_event} for all rounds $t$, $E = \cap_{t = 1}^n E_t$, and $\bar{E}$ be its complement. Then $\prob{\bar{E}} \leq 2|\Sset|n^{-1}$.
\vspace{-0.05in}
\end{lemma}

\begin{proof}
We see that the choice of action given observed context depends on past rounds. This is because the upper confidence bounds depend on which latent states are eliminated, which depend on the history of observed contexts.

For each latent state $s$ and round $t$, let $Y_t(s) = \indicator{B_t = s} (\mu(A_t, X_t) - R_t)$. Observe that in any round $t$, we have $Y_t(s) \mid X_t, H_t$ is $\sigma^2$-sub-Gaussian for any $s$ and round $t$. This implies that $(Y_t(s))_{t \in [n]}$ is a martingale difference sequence with respect to context and history $(X_t, H_t)_{t \in [n]}$, or $\E{}{Y_t(s) \mid X_t, H_t} = 0$ for all rounds $t \in [n]$.

For any round $t$, and state $s \in \Sset$, and any $N_t(s) = u$ for $u < t$, we have the following due to \cref{thm:azuma_general}, 
\begin{align*}
  \prob{\abs{\sum_{\ell=1}^{t-1} Y_t(s)} \geq  \sigma \sqrt{6 u \log n}}
  \leq 2\exp\left[-3\log n\right]
  = 2 n^{-3}\,.
\end{align*}
So, by the union bound, we have
\begin{align*}
  \prob{\bar{E}}
  \leq \sum_{t = 1}^n \sum_{s \in \Sset} \sum_{u = 1}^{t - 1}
  \prob{\abs{\sum_{\ell=1}^{u - 1} Y_t(s)} \geq  \sigma \sqrt{6 u \log n}}
  \leq 2 |\Sset| n^{-1}\,.
\end{align*}
\vspace{-0.05in}
\end{proof}

The first term in Eq. \eqref{eqn:regret_event_decomposition} is small because the probability of $\bar{E}$ is small. Using \cref{thm:ucb_concentration}, and that total regret is bounded by $n$, we have,
$
\E{}{\indicator{\bar{E}} \Regret(n)} \leq n\prob{\bar{E}} \leq 2|\Sset|.
$

For round $t$, the event $\mu(A_{t, *}, X_t) \geq U_t(A_{t, *})$ occurs only if $s_* \notin C_t$ also occurs. By the design of $C_t$ in \mucb, this happens only if $G_t(s_*) > \sigma\sqrt{6 N_t(s)\log n}$. Event $E_t$ says that the opposite is true for all states, including true state $s_*$. So, the second term in Eq. \eqref{eqn:regret_event_decomposition} is at most $0$. 

Now, consider the last term in Eq. \eqref{eqn:regret_event_decomposition}. Let $T_s = \{t \leq n: B_t = s\}$ denote the set of rounds where latent state $s$ is selected. We have, 
\begin{align*}
    \indicator{E} \sum_{t = 1}^n \left(U_t(A_t) - \mu(A_t, X_t)\right)
    &= \indicator{E} \sum_{s \in S} \sum_{t \in T_s}  \left(\mu(A_t, X_t, s) - \mu(A_t, X_t) \right) \\
    &= \indicator{E} \sum_{s \in S} \sum_{t \in T_s}  \left(\mu(A_t, X_t, s) - R_t + R_t - \mu(A_t, X_t) \right) \\
    &\leq \indicator{E} \sum_{s \in S} \left(G_n(s) + \sum_{t \in T_s} \left(R_t - \mu(A_t, X_t) \right)\right) \\
    &\leq \sum_{s \in S} \left(1 + 2\sigma \sqrt{6N_n(s) \log n}\right).
\end{align*}
For the first inequality, we use that the last round $t'$ where state $s$ is selected, we have an upper-bound on the prior gap $G_{t'}(s) \leq \sqrt{6N_{t'}(s)\log n}$. Accounting for the last round yields $G_{n}(s) \leq \sigma \sqrt{6N_n(s)\log n} + 1$. For the last inequality, we use $E$ occurring to bound $\sum_{t \in T_s} \left(R_t - \mu(A_t, X_t) \right) \leq \sigma \sqrt{6N_n(s)\log n}$.

This yields the desired bound on total regret,
\begin{align*}
    \Regret(n)
    &\leq 3|\Sset| + 2\sigma\sqrt{6\log n}\left(\sum_{s \in S} \sqrt{N_n(s)}\right) \\
    &\leq 3|\Sset| + 2\sigma\sqrt{6|\Sset| \log n  \sum_{s \in S}N_n(s)} \\
    &= 3|\Sset| + 2\sigma\sqrt{6|\Sset|n \log n},
\end{align*}
where the last inequality comes from the Cauchy–Schwarz inequality.

\subsection{Proof of \cref{cor:posterior_regret}}

The true latent state $S_* \in \Sset$ is random under Bayes regret. In this case, we still assume that we are given the true model $\theta_*$, so only $S_* \sim P_1$ for known $P_1$. We also have that the optimal action $A_{t, *} = \arg\max_{a \in \Aset} \mu(a, X_t, S_*, \theta_*)$ is random not only due to context, but also $S_*$. 

We define $U_t(a) = \arg\max_{s \in C_t}\mu(a, X_t, S_*, \theta_*)$ as in \mucb. Note the additional randomness due to $S_*$. We can rewrite the Bayes regret as
$
\Bregret(n) 
= \E{}{\Regret(n; S_*, \theta_*)}, 
$
where the outer expectation is over $S_* \sim P_1$. The expression inside the expectation can be decomposed as
\begin{align*}
    \Regret(n, S_*, \theta_*)
    &= \E{}{\indicator{\bar{E}} \sum_{t = 1}^n \mu(A_{t, *}, X_t, S_*) - \mu(A_t, X_t, S_*)}  \\
    &\hspace{-0.5in}+ 
    \E{}{\indicator{E}\sum_{t = 1}^n \left(\mu(A_{t, *}, X_t, S_*) - U_t(A_{t, *})\right)} +
    \E{}{\indicator{E} \sum_{t = 1}^n \left(U_t(A_t) - \mu(A_t, X_t, S_*)\right)}\,,
\end{align*}
where $E, \bar{E}$ are defined as in \cref{sec:thm1_proof}, and we use the decomposition in Eq. \eqref{eqn:posterior_regret_decomposition}.

Each above expression can be bounded exactly as in \cref{thm:ucb_regret}. The reason is that the original upper bounds hold for any $S_*$, and therefore also in expectation over $S_* \sim P_1$. This yields the desired Bayes regret bound.

\subsection{Proof of \cref{thm:ucb_regret_uncertain}}
\label{sec:thm2_proof}
The only difference in the analysis is that we need to incorporate the additional error due to model misspecification.

Let $\mathcal{E} = \{\forall a \in \Aset, x \in \Xset, s \in \Sset: \abs{\hat{\mu}(a, x, s) - \mu(a, x, s)} \leq \varepsilon\}$ be the event that model $\hat{\theta}$ has bounded misspecification and $\bar{\mathcal{E}}$ be its complement. Also let $E$, $\bar{E}$ be defined as in \cref{sec:thm1_proof}.

If $\mathcal{E}$ does not hold, then the best possible upper-bound on regret is $n$; fortunately, we assume in the theorem that the probability of that occurring is bounded by $\delta$. So we can bound the expected $n$-round regret as
\begin{align}
\begin{split}
\label{eqn:regret_event_decomposition_uncertain}
    \Regret(n)
    &= \E{}{\indicator{\bar{\mathcal{E}}} \Regret(n)} + \E{}{\indicator{\bar{E}, \mathcal{E}} \Regret(n)} + 
    \E{}{\indicator{E, \mathcal{E}} \Regret(n)} \\
    &\leq n\delta +
    \E{}{\indicator{\bar{E}, \mathcal{E}} \sum_{t = 1}^n \mu(A_{t, *}, X_t) - \mu(A_t, X_t)}  \\
    &\quad + 
    \E{}{\indicator{E, \mathcal{E}}\sum_{t = 1}^n \left(\mu(A_{t, *}, X_t) - U_t(A_{t, *})\right)} + \E{}{\indicator{E, \mathcal{E}} \sum_{t = 1}^n \left(U_t(A_t) - \mu(A_t, X_t)\right)}\,, \hspace{-0.5in}
\end{split}
\end{align}
where we use the regret decomposition in Eq. \eqref{eqn:ucb_regret_decomposition}.

The second term in Eq. \eqref{eqn:regret_event_decomposition_uncertain} is small because the probability of $\bar{E}$ is small. Using \cref{thm:ucb_concentration}, and that total regret is bounded by $n$, we have,
$
\E{}{\indicator{\bar{E}, \mathcal{E}}\Regret(n)} \leq n\prob{\bar{E}} \leq 2|\Sset|.
$

If $\mathcal{E}$ occurs, the event $\mu(A_{t, *}, X_t) - U_t(A_{t, *}) \geq \varepsilon$ for any round $t$ occurs only if $s_* \notin C_t$ also occurs. By the design of $C_t$ in \mmucb,
this happens if $G_t(s_*) \geq \sigma\sqrt{6 N_t(s)\log n}$. Since
\begin{align*}
G_t(s_*) = \sum_{\ell = 1}^{t-1} 
    \indicator{B_\ell = s_*} \left(\hat{\mu}(A_\ell, X_\ell) - \varepsilon - R_\ell\right)
    \leq \sum_{\ell = 1}^{t-1} 
    \indicator{B_\ell = s_*} \left(\mu(A_\ell, X_\ell) - R_\ell\right),
\end{align*}
we see that event $E_t$ says that the opposite is true for all states, including true state $s_*$. Hence, the third term in Eq. \eqref{eqn:regret_event_decomposition_uncertain} is bounded by $n\varepsilon$. 

Now, consider the last term in Eq. \eqref{eqn:regret_event_decomposition_uncertain}. Let $T_s = \{t \leq n: B_t = s\}$ denote the set of rounds where latent state $s$ is selected. We have, 
\begin{align*}
    \indicator{E, \mathcal{E}}\sum_{t = 1}^n \left(U_t(A_t) - \mu(A_t, X_t)\right)
    &= \indicator{E, \mathcal{E}}\sum_{s \in S} \sum_{t \in T_s}  \left(\hat{\mu}(A_t, X_t, s) - \mu(A_t, X_t) \right) \\
    &= n\varepsilon + \indicator{E, \mathcal{E}} \sum_{s \in S} \sum_{t \in T_s} \left(\hat{\mu}(A_t, X_t, s) - \varepsilon - R_t + R_t - \mu(A_t, X_t) \right) \\
    &\leq n\varepsilon + \indicator{E, \mathcal{E}} \sum_{s \in S} \left(G_n(s) + \sum_{t \in T_s} \left(R_t - \mu(A_t, X_t) \right)\right) \\
    &\leq n\varepsilon + \sum_{s \in S} \left(1 + 2\sigma \sqrt{6N_n(s) \log n}\right)\,.
\end{align*}
For the first inequality, we use that the last round $t'$ where state $s$ is selected, we have an upper-bound on the prior gap $G_{t'}(s) \leq \sqrt{6N_{t'}(s)\log n}$. Accounting for the last round yields $G_{n}(s) \leq \sigma \sqrt{6N_n(s)\log n} + 1$. For the last inequality, we use $E$ occurring to bound $\sum_{t \in T_s} \left(R_t - \mu(A_t, X_t) \right) \leq \sigma \sqrt{6N_n(s)\log n}$.

This yields the desired bound on total regret,
\begin{align*}
    \Regret(n)
    &\leq n\delta + 3|\Sset| 
    + 2n\varepsilon 
    + 2\sigma\sqrt{6\log n}\left(\sum_{s \in S} \sqrt{N_n(s)}\right) \\
    &\leq n\delta +  3|\Sset|  + 2n\varepsilon + 
    2\sigma\sqrt{6|\Sset| \log n  \sum_{s \in S}N_n(s)} \\
    &= n\delta +  3|\Sset|  +
    2n\varepsilon +
    2\sigma\sqrt{6|\Sset|n \log n},
\end{align*}
where the last inequality comes from the Cauchy–Schwarz inequality.

\subsection{Proof of \cref{cor:posterior_regret_uncertain}}

Both latent state $S_* \in \Sset$ and model $\theta_* \in \Theta$ are random, and drawn as $S_*, \theta_* \sim P_1$, where the prior $P_1$ is known. In this case, the true model $\theta_*$ is not known to us.

Using marginalized means $\bar{\mu}(a, x, s)$, and $\varepsilon, \delta > 0$ as defined in the statement of the corollary, we write,
\begin{align*}
G_t(s) = \sum_{\ell = 1}^{t-1} \indicator{B_\ell = s}
  \left(\bar{\mu}(A_\ell, X_\ell, s) - \varepsilon - R_\ell\right),
\end{align*}
and $U_t(a) = \arg\max_{s \in C_t}\bar{\mu}(a, X_t, s)$. This is in contrast to $G_t(s)$ and $U_t(a)$ in \mmucb, which use $\hat{\mu}(a, x, s)$ from a single model. Conceptually though, both $\hat{\mu}(a, x, s)$ and $\bar{\mu}(a, x, s)$ are just $\varepsilon$-close point estimates of $\mu(a, x, s)$ due to the assumptions made about the true model $\theta_*$ in the theorem and corollary, respectively. 

We can rewrite the Bayes regret as
$
\Bregret(n) 
= \E{}{\Regret(n; S_*, \theta_*)}, 
$
where the outer expectation is over $S_*, \theta_* \sim P_1$. The expression inside the expectation can be written as,
\begin{align*}
    \Regret(n; S_*, \theta_*)
    &\leq n\delta + \E{}{\indicator{\bar{E}, \mathcal{E}} \sum_{t = 1}^n \mu(A_{t, *}, X_t, S_*, \theta_*) - \mu(A_t, X_t, S_*, \theta_*)}  \\
    &\hspace{-0.8in}+ 
    \E{}{\indicator{E, \mathcal{E}}\sum_{t = 1}^n \left(\mu(A_{t, *}, X_t, S_*, \theta_*) - U_t(A_{t, *})\right)} +
    \E{}{\indicator{E, \mathcal{E}} \sum_{t = 1}^n \left(U_t(A_t) - \mu(A_t, X_t, S_*, \theta_*)\right)}\,,
\end{align*}
where $\mathcal{E}, E, \bar{E}$ are defined as in \cref{sec:thm2_proof}, and we use the decomposition in Eq. \eqref{eqn:posterior_regret_decomposition}.

The expressions can be bounded exactly as in \cref{thm:ucb_regret_uncertain}. The upper bound is worst-case and holds for any $S_*, \theta_*$, and thus also holds after taking an expectation over the prior $S_*, \theta_* \sim P_1$. This bounds the Bayes regret, as desired.

\end{document}